\documentclass[12pt]{myjmlr}
\usepackage{times,wrapfig,amsmath,amsfonts,bm,color,enumitem,algorithm,algpseudocode}

\newcommand{\inner}[2]{\left\langle #1, #2 \right\rangle}

\newcommand{\E}{\ensuremath{\mathbb{E}}}

\newcommand{\norm}[1]{\left\lVert{#1}\right\rVert}
\newcommand{\abs}[1]{\left\lvert{#1}\right\rvert}

\newcommand{\sign}{\operatorname{sign}}
\newcommand{\conv}{\operatorname{conv}}

\newcommand{\R}{\mathbb{R}}

\newcommand{\HH}{\mathcal{H}}
\newcommand{\GG}{\mathcal{G}}

\newcommand{\defeq}{\triangleq}

\mathchardef\hyphen="2D
\usepackage{amsmath}
\usepackage{hyperref}
\usepackage{enumitem}
\usepackage{url}
\usepackage{times}

\makeatletter
\newtheorem*{rep@theorem}{\rep@title}
\newcommand{\newreptheorem}[2]{%
\newenvironment{rep#1}[1]{%
 \def\rep@title{#2 \ref{##1}}%
 \begin{rep@theorem}}%
 {\end{rep@theorem}}}
\makeatother

\newcommand{\calF}{\mathcal{F}}

\newcommand{\calR}{\mathcal{R}}
\newcommand{\calG}{\mathcal{G}}
\newcommand{\calL}{\mathcal{L}}
\newcommand{\calN}{\mathcal{N}}
\newcommand{\calX}{\mathcal{X}}

\newcommand{\gpath}{\text{path}}
\newcommand{\layer}{\text{layer}}
\newcommand{\DAG}{\text{DAG}}
\newcommand{\dout}{d_\text{out}}
\newcommand{\din}{d_\text{in}}
\newtheorem{claim}{Claim}
\newenvironment{sketch}{\paragraph{Proof sketch}\mbox{}}{\hfill\BlackBox}
\newcommand{\vin}{v_{\textrm{in}}}
\newcommand{\vout}{v_{\textrm{out}}}
\newcommand{\relu}{\sigma_{\textsc{relu}}}
\newcommand{\gnorm}[1]{\norm{#1}}

\newcommand{\pathr}{\phi}
\newcommand{\convexnn}{\nu}

\newreptheorem{theorem}{Theorem}
\newreptheorem{lemma}{Lemma}

\newcommand{\removed}[1]{}
\title{Norm-Based Capacity Control in Neural Networks}
 \author{\Name{Behnam Neyshabur} \Email{bneyshabur@ttic.edu}\\
 \Name{Ryota Tomioka} \Email{tomioka@ttic.edu}\\
  \Name{Nathan Srebro} \Email{nati@ttic.edu}\\
  \addr Toyota Technological Institute at Chicago\\
Chicago, IL 60637, USA \\}
\begin{document}
\maketitle
\begin{abstract}
We investigate the capacity, convexity and characterization of a
general family of norm-constrained feed-forward networks.
\end{abstract}

\begin{keywords}
Feed-forward neural networks, deep learning, scale-sensitive capacity control
\end{keywords}

\section{Introduction}
The statistical complexity, or capacity, of {\em unregularized}
feed-forward neural networks, as a function of the network size and
depth, is fairly well understood.  With hard-threshold activations,
the VC-dimension, and hence sample complexity, of the class of
functions realizable with a feed-forward network is equal, up to
logarithmic factors, to the number of edges in the
network \citep{anthony09,shalev14}, corresponding to the number of parameters.  With continuous activation functions the VC-dimension
could be higher, but is fairly well understood and is still controlled
by the size and depth of the network.\footnote{Using weights
with very high precision and vastly different magnitudes it is
possible to shatter a number of points quadratic in the number of
edges when activations such as the sigmoid, ramp or hinge are used
\citep[Chapter 20.4]{shalev14}.  But even with such
activations, the VC dimension can still be bounded by the size and
depth \citep{bartlett98,anthony09,shalev14}.}

But feedforward networks are often trained with some
kind of explicit or implicit regularization, such as weight decay,
early stopping, ``max regularization'', or more exotic
regularization such as drop-outs.  What is the effect of such
regularization on the induced hypothesis class and its capacity?

For linear prediction (a one-layer feed-forward network) we know that
using regularization the capacity of the class can be bounded only in
terms of the norms, with no (or a very weak) dependence on the number
of edges (i.e.~the input dimensionality or
number of linear coefficients).  E.g., we understand very well how the
capacity of $\ell_2$-regularized linear predictors can be bounded in
terms of the norm alone (when the norm of the data is also bounded),
even in infinite dimension.

A central question we ask is: can we bound the capacity of
feed-forward network in terms of norm-based regularization alone,
without relying on network size and even if the network size (number
of nodes or edges) is unbounded or infinite?  What type of
regularizers admit such capacity control?  And how does the capacity
behave as a function of the norm, and perhaps other network
parameters such as depth?

Beyond the central question of capacity control, we also analyze the
convexity of the resulting hypothesis class---unlike unregularized
size-controlled feed-forward networks, infinite magnitude-controlled
networks have the potential of yielding convex hypothesis classes
(this is the case, e.g., when we move from rank-based control on
matrices, which limits the number of parameters to magnitude based
control with the trace-norm or max-norm).  A convex class might be
easier to optimize over and might be convenient in other ways.

In this paper we focus on networks with rectified linear units and two
natural types of norm regularization: bounding the norm of the
incoming weights of each unit (per-unit regularization) and bounding
the overall norm of all the weights in the system jointly (overall
regularization, e.g.~limiting the overall sum of the magnitudes, or
square magnitudes, in the system).  We generalize both of these with a
single notion of group-norm regularization: we take the $\ell_p$
norm over the weights in each unit and then the $\ell_q$ norm over
units.  In Section \ref{sec:group} we present this regularizer and
obtain a tight understanding of when it provides for size-independent
capacity control and a characterization of when it induces convexity.
We then apply these generic results to per-unit regularization
(Section \ref{sec:path}) and overall regularization (Section
\ref{sec:overall}), noting also other forms of regularization that are
equivalent to these two.  In particular, we show how per-unit
regularization is equivalent to a novel path-based regularizer
and how overall $\ell_2$ regularization for two-layer networks is
equivalent to so-called ``convex neural networks''
\citep{Bengio05}.  In terms of capacity control, we show that
per-unit regularization allows size-independent capacity-control only
with a per-unit $\ell_1$-norm, and that overall $\ell_p$
regularization allows for size-independent capacity control only when
$p \leq 2$, even if the depth is bounded.  In any case, even if we
bound the sum of all magnitudes in the system, we show that an
exponential dependence on the depth is unavoidable.

As far as we are aware, prior work on size-independent capacity
control for feed-forward networks considered only per-unit $\ell_1$
regularization, and per-unit $\ell_2$ regularization for two-layered
networks (see discussion and references at the beginning of Section
\ref{sec:path}).  Here, we extend the scope significantly, and
provide a broad characterization of the types of regularization
possible and their properties.  In particular, we consider overall
norm regularization, which is perhaps the most natural form of
regularization used in practice (e.g.~in the form of weight decay).
We hope our study will be useful in thinking about, analyzing and
designing learning methods using feed-forward networks.  Another
motivation for us is that complexity of large-scale optimization is
often related to scale-based, not dimension-based complexity.
Understanding when the scale-based complexity depends exponentially on
the depth of a network might help shed light on understanding the
difficulties in optimizing deep networks.

\section{Preliminaries: Feedforward Neural Networks}

A feedforward neural network that computes a function
$f\!:\!\R^D\rightarrow\R$ is specified by a directed acyclic graph (DAG)
$G(V,E)$ with $D$ special ``input nodes'' $\vin[1],\ldots,\vin[D]\in
V$ with no incoming edges and a special ``output node'' $\vout\in V$
with no outgoing edges, weights $w\!:\!E\rightarrow \R$ on the edges, and
an {\em activation function} $\sigma\!:\!\R\rightarrow\R$.  

Given an input $x\in\R^D$, the output values of the input units are set
to the coordinates of $x$, $o(\vin[i])=x[i]$ (we might want
to also add a special ``bias'' node with
$o(\vin[0])=1$, or just rely on the inputs having a fixed ``bias
coordinate''), the output value of internal nodes (all nodes except
the input and output nodes) are defined according to the forward
propagation equation:
\begin{equation}
  \label{eq:forward}
  o(v)=\sigma\left( \sum_{(u\rightarrow v)\in E} w(u\rightarrow v) o(u)\right),
\end{equation}
and the output value of the output unit is defined as
$o(\vout)=\sum_{(u\rightarrow \vout)\in E} w(u\rightarrow \vout)
o(u)$.  The network is then said to compute the function
$f_{G,w,\sigma}(x)=o(\vout)$.  Given a graphs $G$ and activation
function $\sigma$, we can consider the hypothesis class of functions
$\calN^{G,\sigma}=\{ f_{G,w,\sigma}\!:\!\R^D\rightarrow\R \;|\;
w\!:\!E\rightarrow \R \}$ computable using some setting of the weights.

We will refer to the {\em size} of the network, which is the overall
number of edges $\abs{E}$, the {\em depth} $d$ of the network, which
is the length of the longest directed path in $G$, and the {\em
  in-degree} (or width) $H$ of a network, which is the maximum
in-degree of a vertex in $G$.

A special case of feedforward neural networks are layered fully
connected networks where vertices are partitioned into layers and
there is a directed edge from every vertex in layer $i$ to every
vertex in layer $i+1$. We index the layers from the first layer, $i=1$
whose inputs are the input nodes, up to the last layer $i=d$ which
contains the single output node---the number of layers is thus equal
to the depth and the in-degree is the maximal layer size.  We denote
by $\layer(d,H)$ the layered fully connected network with $d$ layers
and $H$ nodes per layer (except the output layer that has a single
node), and also allow $H=\infty$.  We will also use the shorthand
$\calN^{d,H,\sigma}=\calN^{\layer(d,H),\sigma}$ and
$\calN^{d,\sigma}=\calN^{\layer(d,\infty),\sigma}$.

Layered networks can be parametrized by a sequence of matrices $W_1
\in \R^{H\times D}, W_2,$ $W_3, \ldots, W_{d-1} \in \R^{H\times H}, W_d
\in \R^{1 \times H}$ where the row $W_i[j,:]$ contains the input weights
to unit $j$ in layer $i$, and
\begin{equation}
  \label{eq:layered}
f_W(x) = W_d \sigma (W_{d-1} \sigma( W_{d-2} ( \ldots \sigma({W_{1}} x)))),
\end{equation}
where $\sigma$ is applied element-wise.

We will focus mostly on the hinge, or RELU (REctified Linear Unit)
activation, which is currently in popular use \citep{nair10,glorot11,zeiler13}, $\relu(z) = [z]_+ = \max(z,0)$.
When the activation will not be specified, we will
implicitly be referring to the RELU.  The RELU has several convenient
properties which we will exploit, some of them shared with other
activation functions:
\begin{description}\itemsep1pt \parskip0pt \parsep0pt
\item[Lipshitz] The hinge is Lipschitz continuous with Lipshitz
  constant one.  This property is also shared by the sigmoid and the
  ramp activation $\sigma(z)=\min(\max(0,z),1)$.
\item[Idempotency] The hinge is idempotent,
  i.e.~$\relu(\relu(z))=\relu(z)$. This property is also shared by the ramp
  and hard threshold activations.
\item[Non-Negative Homogeneity] For a non-negative scalar $c \geq 0$
  and any input $z\in\R$ we have $\relu(c\cdot z)=c\cdot \relu(z)$.
  This property is important as it allows us to scale the incoming
  weights to a unit by $c>0$ and scale the outgoing edges by $1/c$
  without changing the the function computed by the network.  For
  layered graphs, this means we can scale $W_i$ by $c$ and compensate
  by scaling $W_{i+1}$ by $1/c$.
\end{description}

We will consider various measures $\alpha(w)$ of the magnitude of the
weights $w(\cdot)$.  Such a measure induces a complexity measure on functions $f\in\calN^{G,\sigma}$ defined by $\alpha^{G,\sigma}(f)=\inf_{f_{G,w,\sigma}=f} \alpha(w)$.
The sublevel sets of the complexity measure $\alpha^{G,\sigma}$ form a family of hypothesis classes
$\calN^{G,\sigma}_{\alpha\leq a} = \{ f\in\calN^{G,\sigma} \;|\;
\alpha^{G,\sigma}(f)\leq a \}$.  Again we will use the shorthand
$\alpha^{d,H,\sigma}$ and $\alpha^{d,\sigma}$ when referring to layered
graphs $\layer(d,H)$ and $\layer(d,\infty)$ respectively, and frequently drop
$\sigma$ when RELU is implicitly meant.

For binary function $g:\{\pm 1\}^D \rightarrow {\pm 1}$ we say that $g$ is
realized by $f$ with unit margin if $\forall_x f(x)g(x)\geq 1$.  A set
of points $S$ is shattered with unit margin by a hypothesis class
$\calN$ if all $g:S \rightarrow {\pm 1}$ can be realized with unit
margin by some $f\in\calN$.

\section{Group Norm Regularization}
\label{sec:group}
\removed{
The group norm will be the main generic regularizer we will consider, as it captures as special cases other regularizers of interest. For any DAG $G$ we can define the following measure:
$$
\overline{\mu}_{p,q}(G,w) = \Biggl(\sum_{v \in
  V}\left(\sum_{(u\rightarrow v) \in E} \left\lvert w(u\rightarrow
v)\right\rvert ^p\right)^{q/p}\Biggr)^{1/q}
$$
For any graph family $\calG(d)$ with depth $d$ and any function $g:\calX\rightarrow \R$, group-complexity ($\overline{\gamma}^\calG_{p,q}$) is defined as:
\begin{align*}
\overline{\gamma}^{\calG(d)}_{p,q}(g) &= \min_{f_{G,w}\in \calN^{\calG(d)} \atop f_{G,w}=g } \frac{\overline{\mu}_{p,q}(G,w)}{d}
\end{align*}

It is often more convenient to simplify the calculations by considering the layered networks. In fact, in the section~\ref{sec:path}, we show that if a function is realizable with a DAG, it is also realizable with a layered network of the same depth. In a layered graph, scaling up the weights of one layer by $\alpha$ and  scaling down the weights of another layer by $\frac{1}{\alpha}$ gives you a network that realized the same function. Therefore it is more national to consider the product of norms:
$$
\gamma_{p,q}^{\layer(d)}(g) =  \min_{f_W \in \calN^{\layer(d)}  \atop f_{W}=g} \mu^{\layer(d)}_{p,q}(W)= \min_{f_W \in \calN^{\layer(d)} \atop f_{W}=g}\prod_{k=1}^d \norm{W_k}_{p,q}
$$
}

Considering the grouping of weights going into each edge of the
network, we will consider the following generic group-norm type
regularizer, parametrized by $1\leq p,q \leq\infty$:
\begin{equation}
  \label{eq:mu}
  \mu_{p,q}(w) = \left(\sum_{v \in V}\left(\sum_{(u\rightarrow v) \in E} \left\lvert w(u\rightarrow v)\right\rvert ^p\right)^{q/p}\right)^{1/q}.
\end{equation}
Here and elsewhere we allow $q=\infty$ with the usual conventions that
$(\sum z_i^q)^{1/q}=\sup z_i$ and $1/q=0$ when it appears in other
contexts.  When $q=\infty$ the group regularizer \eqref{eq:mu} imposes
a per-unit regularization, where we constrain the norm of the incoming
weights of each unit separately, and when $q=p$ the regularizer
\eqref{eq:mu} is an ``overall'' weight regularizer, constraining the
overall norm of all weights in the system.  E.g., when $q=p=1$ we are
paying for the sum of all magnitudes of weights in the network, and
$q=p=2$ corresponds to overall weight-decay where we pay for the sum
of square magnitudes of all weights (i.e.~the overall Euclidean norm
of the weights).

For a layered graph, we have:
\begin{align}
\mu_{p,q}(W) &= \left(\sum_{k=1}^d\sum_{i=1}^H \left(
\sum_{j=1}^H\abs{W_k[i,j]}^p
\right)^{q/p}\right)^{1/q} 
\!\!=d^{1/q} \left(\frac{1}{d} \sum_{k=1}^d
\gnorm{W_k}^q_{p,q}\right)^{1/q} \notag \\ 
&\geq d^{1/q} \left( \prod_{k=1}^d \gnorm{W_k}_{p,q} \right)^{1/d} \defeq
d^{1/q} \sqrt[d]{\gamma_{p,q}(W)} \label{eq:mugeqgamma}
\end{align}
where $\displaystyle \gamma_{p,q}(W) = \prod_{k=1}^d \gnorm{W_k}_{p,q}$
aggregates the layers by multiplication instead of summation.
The inequality~\eqref{eq:mugeqgamma} holds regardless of the
activation function, and so for any $\sigma$ we
have:
\begin{equation}
  \label{eq:mugammaforf}
  \gamma_{p,q}^{d,H,\sigma}(f)\leq \left(\frac{\mu^{d,H,\sigma}(f)_{p,q}}{d^{1/q}}\right)^d.
\end{equation}
But due to the homogeneity of the RELU activation, when this
activation is used we can always balance the norm between the
different layers without changing the computed function so as to
achieve equality in \eqref{eq:mugeqgamma}:
\begin{claim}\label{clm:mugamma}
For any $f\in\calN^{d,H,\relu}$, 
$\displaystyle \mu^{d,H,\relu}_{p,q}(f) = d^{1/q}
\sqrt[d]{\gamma_{p,q}^{d,H,\relu}(f)}$.
\end{claim}
\begin{proof}
  Let $W$ be weights that realizes $f$ and are optimal with respect to
  $\gamma_{p,g}$; i.e.~$\gamma_{p,q}(W) = \gamma^{d,H}_{p,q}(f)$.  Let
  $\widetilde{W}_k = \sqrt[d]{\gamma_{p,q}(W)}W_k/\norm{W_k}_{p,q}$,
  and observe that they also realize $f$.  We now have:
\begin{align}
\mu^{d,H}_{p,q}(f) \leq \mu_{p,q}(\widetilde{W}) = \Bigl(\sum\nolimits_{k=1}^d \norm{\widetilde{W}_k}^q_{p,q}\Bigr)^{1/q}
= \Bigl(d\Bigl(\gamma_{p,q}(W)\Bigr)^{q/d}\Bigr)^{1/q} = d^{1/q}\sqrt[d]{\gamma^{d,H,\relu}_{p,q}(f)}\notag
\end{align}
which together with \eqref{eq:mugeqgamma} completes the proof.
\end{proof}
The two measures are therefore equivalent when we use RELUs, and
define the same level sets, or family of hypothesis classes, which we
refer to simply as $\calN^{d,H}_{p,q}$.  In the remainder of this
Section, we investigate convexity and generalization properties of
these hypothesis classes.

\subsection{Generalization and Capacity}

In order to understand the effect of the norm on the sample
complexity, we bound the Rademacher complexity of the classes
$\calN^{d,H}_{p,q}$.  Recall that the Rademacher Complexity is
a measure of the capacity of a hypothesis class on a specific sample,
which can be used to bound the difference between empirical and
expected error, and thus the excess generalization error of empirical
risk minimization (see, e.g., \cite{bartlett03} for a
complete treatment, and Appendix \ref{sec:rademacher} for the exact definitions we
use).  In particular, the Rademacher complexity typically scales as
$\sqrt{C/m}$, which corresponds to a sample complexity of
$O(C/\epsilon^2)$, where $m$ is the sample size and $C$ is the
effective measure of capacity of the hypothesis class.

\begin{theorem}\label{thm:l-norm}
For any $d,q\geq 1$, any $1\leq p <\infty$ and any set $S=\{x_1,\dots,x_m\}\subseteq\R^D$:
\begin{align*}
\calR_m(\calN_{\gamma_{p,q}\leq \gamma}^{d,H,\relu}) &\leq 
\gamma \left( 2 H^{[\frac{1}{p^*} -
        \frac{1}{q}]_+}\right)^{(d-1)} \calR^{\text{linear}}_{m,p,D}\\
&\leq \sqrt{ \frac{\gamma^2 \left( 2 H^{[\frac{1}{p^*} -
        \frac{1}{q}]_+}\right)^{2(d-1)}\min\{p^*,4\log(2D)\} \max_i \norm{x_i}_{p^*}^2}{m}}
\end{align*}
and so:
\begin{align*}
\calR_m(\calN_{\mu_{p,q}\leq \mu}^{d,H,\relu}) &\leq 
\mu^{d} \left( 2 H^{[\frac{1}{p^*} -
        \frac{1}{q}]_+} /\sqrt[q]{d}\right)^{(d-1)}\calR^{\text{linear}}_{m,p,D}\\
&\leq \sqrt{ \frac{\mu^{2d} \left( 2 H^{[\frac{1}{p^*} -
        \frac{1}{q}]_+} /\sqrt[q]{d}\right)^{2(d-1)}\min\{p^*,4\log(2D)\} \max_i \norm{x_i}_{p^*}^2}{m}}     
\end{align*}
where the second inequalities hold only if $1\leq p \leq 2$, $\calR^{\text{linear}}_{m,p,D}$ is the Rademacher complexity of $D$-dimensional linear predictors with unit $\ell_p$ norm with respect to a set of $m$ samples and $p^*$ is such that $\frac{1}{p^*} + \frac{1}{p}=1$.
\end{theorem}
\begin{sketch}
We prove the bound by induction, showing that for any $q,d>1$ and $1\leq p < \infty$,
$$
\calR_m(\calN_{\gamma_{p,q}\leq \gamma}^{d,H,\relu}) \leq 2 H^{[\frac{1}{p^*} -\frac{1}{q}]_+}\calR_m(\calN_{\gamma_{p,q}\leq \gamma}^{d-1,H,\relu}).
$$
The intuition is that when $p^*<q$, the Rademacher complexity increases
 by simply distributing the weights among neurons and if $p^*\geq q$
 then the supremum is attained when the output neuron is connected to a neuron with highest Rademacher complexity in the lower layer and all other weights in the top layer are set to zero. For a complete proof, see Appendix \ref{sec:rademacher}.
\end{sketch}

Note that for $2\leq p < \infty$, the bound on the Rademacher complexity scales with $m^{\frac{1}{p}}$ (see section \ref{sec:linear} in appendix) because:
\begin{equation}
\calR^{\text{linear}}_{m,p,D} \leq \frac{\sqrt{2}\norm{X}_{2,p^*}}{m} \leq \frac{\sqrt{2}\max_i\norm{x_i}_{p^*}}{m^{\frac{1}{p}}}
\end{equation}
The bound in Theorem \ref{thm:l-norm} depends on both the magnitude
of the weights, as captured by $\mu_{p,q}(W)$ or $\gamma_{p,q}(W)$, and also on
the width $H$ of the network (the number of nodes in each layer).
However, the dependence on the width $H$ disappears, and the bound
depends only on the magnitude, as long as $q \leq p^*$
(i.e. $1/p+1/q\geq 1$). This happens, e.g., for overall $\ell_1$ and $\ell_2$ regularization, for per-unit $\ell_1$ regularization, and whenever $1/p+1/q=1$.
In such cases, we can omit the size constraint and state the theorem for an infinite-width layered network (i.e.~a network with an infinitely countable number of units, when the number of units is allowed to be as large as needed):
\begin{corollary}\label{cor:noH}
For any $d\geq 1$, $1\leq p < \infty$ and $1\leq q\leq p^*=p/(p-1)$, and any set $S=\{x_1,\dots,x_m\}\subseteq\R^D$,
\begin{align*}
\calR_m(\calN_{\gamma_{p,q}\leq \gamma}^{d,H,\relu}) &\leq 
\gamma 2^{(d-1)} \calR^{\text{linear}}_{m,p,D}\\
&\leq \sqrt{ \frac{\gamma^2 \left( 2 H^{[\frac{1}{p^*} -
        \frac{1}{q}]_+}\right)^{2(d-1)}\min\{p^*,4\log(2D)\} \max_i \norm{x_i}_{p^*}^2}{m}}
\end{align*}
and so:
\begin{align*}
\calR_m(\calN_{\mu_{p,q}\leq \mu}^{d,H,\relu}) &\leq 
\left( 2 \mu /\sqrt[q]{d}\right)^{d}\calR^{\text{linear}}_{m,p,D}\\
&\leq \sqrt{ \frac{ \left( 2 \mu /\sqrt[q]{d}\right)^{2d} \min\{p^*,4\log(2D)\} \max_i \norm{x_i}_{p^*}^2}{m}}     
\end{align*}
where the second inequalities hold only if $1\leq p \leq 2$ and $\calR^{\text{linear}}_{m,p,D}$ is the Rademacher complexity of $D$-dimensional linear predictors with unit $\ell_p$ norm with respect to a set of $m$ samples.
\end{corollary}

\subsection{Tightness}\label{sec:tight}
We next investigate the tightness of the complexity bound in Theorem
\ref{thm:l-norm}, and show that when $1/p+1/q<1$ the dependence on the
width $H$ is indeed unavoidable.  We show not only that the bound on
the Rademacher complexity is tight, but that the implied bound on the
sample complexity is tight, even for binary classification with a
margin over binary inputs.  To do this, we show how we can shatter the
$m=2^D$ points $\{\pm 1\}^D$ using a network with small group-norm:

\begin{theorem}\label{thm:shattering}
For any $p,q \geq 1$ (and $1/p^*+1/p=1$) and any depth $d\geq 2$, the
$m=2^D$ points $\{ \pm 1 \}^D$ can be shattered with unit margin by
$\calN^{d,H}_{\gamma_{p,q}\leq\gamma}$ with:
$$
\gamma \leq D^{1/p} \, m^{1/p+1/q} \, H^{-(d-2)[1/p^*-1/q]_+}
$$
\end{theorem}
\begin{proof}
Consider a size $m$ subset $S_m$ of $2^D$ vertices of the $D$
 dimensional hypercube  $\{-1,+1\}^D$. We construct the first layer
 using $m$ units. Each unit 
has a unique weight vector consisting of $+1$ and $-1$'s and will output
a positive value if and only if the sign pattern of the input $x\in S_m$ matches
that of the weight vector. The second layer has a single unit and
connects to all $m$ units in the first layer. For any $m$
dimensional sign pattern $b\in\{-1,+1\}^{m}$, we can choose the
weights of the second layer to be $b$, and the network will output the
desired sign for each $x\in S_m$ with unit margin. The norm
 of the network is at most
$
(m\cdot D^{q/p})^{1/q}\cdot m^{1/p}=D^{1/p}\cdot m^{(1/p+1/q)}.
$
This establishes the claim for $d=2$. For $d>2$ and $1/p+1/q\geq 1$, we
 obtain the same norm and unit margin by adding $d-2$ layers with one
 unit in each layer connected to the previous layer by a unit weight.
 For $d>2$ and $1/p+1/q<1$, we show
the dependence on $H$ by recursively replacing the top unit with $H$
copies of it and adding an averaging unit on top of that. More
specifically, given the above $d=2$ layer network, we make $H$ copies of
the output unit with rectified linear activation and add a 3rd layer with
one output unit with uniform weight $1/H$ to all the copies in the 2nd
layer. Since this operation does not change the output of the network,
we have the same margin and now the norm of the network is
$
(m\cdot D^{q/p})^{1/q}\cdot (Hm^{q/p})^{1/q}\cdot (H(1/H^p))^{1/p}
=D^{1/p}\cdot m^{(1/p+1/q)}\cdot H^{1/q-1/p^*}.
$
That is, we have reduced the norm by factor $H^{1/q-1/p^*}$. By repeating
 this process, we get the geometric reduction in the norm
 $H^{(d-2)(1/q-1/p^*)}$, which concludes the proof.
\end{proof}

To understand this lower bound, first consider the bound without the
dependence on the width $H$.  We have that for any depth $d\geq 2$,
$\gamma \leq m^r D = m^r \log m$ (since $1/p\leq 1$ always) where
$r=1/p+1/q\leq 2$.  This means that for any depth $d\geq 2$ and any
$p,q$ the sample complexity of learning the class scales as
$m=\Omega(\gamma^{1/r}/\log \gamma) \geq
\tilde{\Omega}(\sqrt{\gamma})$.  This shows a polynomial dependence on
$\gamma$, though with a lower exponent than the $\gamma^2$ 
(or higher for $p>2$) dependence
in Theorem \ref{thm:l-norm}.  Still, if we now consider the complexity
control as a function of $\mu_{p,q}$ we get a sample complexity of at
least $\Omega(\mu^{d/2}/\log \mu)$, establishing that if we control
the group-norm as in \eqref{eq:mu}, we cannot avoid a sample
complexity which depends exponentially on the depth.  Note that in our
construction, all other factors in Theorem \ref{thm:l-norm}, namely
$\max_i\norm{x_i}$ and $\log D$, are logarithmic (or double-logarithmic) in
$m$.

Next we consider the dependence on the width $H$ when $1/p+1/q<1$.
Here we have to use depth $d\geq 3$, and we see that indeed as the
width $H$ and depth $d$ increase, the magnitude control $\gamma$ can
decrease as $H^{(1/p^*-1/q)(d-2)}$ without decreasing the capacity,
matching Theorem 1 up to an offset of 2 on the depth.  In particular,
we see that in this regime we can shatter an arbitrarily large number
of points with arbitrarily low $\gamma$ by using enough hidden units,
and so the capacity of $\calN^d_{p,q}$ is indeed infinite and it
cannot ensure any generalization.

\subsection{Convexity}

Finally we establish a sufficient condition for the hypothesis classes
$\calN^d_{p,q}$ to be convex.  We are referring to convexity of the
functions in the $\calN^d_{p,q}$ independent of a specific
representation.  If we consider a, possibly regularized, empirical
risk minimization problem on the weights, the objective (the empirical
risk) would never be a convex function of the weights (for depth
$d\geq 2$), even if the regularizer is convex in $w$ (which it always
is for $p,q\geq 1$).  But if we do not bound the width of the network,
and instead rely on magnitude-control alone, we will see that the
resulting hypothesis class, and indeed the complexity measure, may be
convex (with respect to taking convex combinations of functions, {\em
  not} of weights).

\begin{theorem}\label{thm:cvx}
  For any $d,p,q \geq 1$ such that $\frac{1}{q}\leq
  \frac{1}{d-1}\big(1-\frac{1}{p}\big)$, $\gamma^d_{p,q}(f)$ is a semi-norm
  in $\calN^d$.
\end{theorem}
In particular, under the condition of the Theorem, $\gamma^d_{p,q}$ is
convex, and hence its sublevel sets $\calN^d_{p,q}$ are convex, and so
$\mu^d_{p,q}$ is quasi-convex (but not convex). \removed{ The condition holds
for per-unit regularization ($q=\infty$) for any $p\geq 1$, and for
overall regularization ($q=p$) whenever $p=q \geq d$.  However, }

\begin{sketch}
  To show convexity, consider two functions
  $f,g\in\calN^d_{\gamma_{p,q}\leq \gamma}$ and $0<\alpha<1$, and
  let $U$ and $V$ be the weights realizing $f$ and $g$ respectively
  with $\gamma_{p,q}(U)\leq \gamma$ and $\gamma_{p,q}(V)\leq\gamma$.  We will construct
  weights $W$ realizing $\alpha f+(1-\alpha)g$ with
  $\gamma_{p,q}(W)\leq \gamma$.  This is done by first balancing $U$
  and $V$ s.t.~at each layer
  $\norm{U_i}_{p,q}=\sqrt[d]{\gamma_{p,q}(U)}$ and $\norm{V_i}_{p,q}=\sqrt[d]{\gamma_{p,q,}(V)}$ and then
  placing $U$ and $V$ side by side, with no interaction between the
  units calculating $f$ and $g$ until the output layer.  The output
  unit has weights $\alpha U_d$ coming in from the $f$-side and
  weights $(1-\alpha)V_d$ coming in from the $g$-side.  In Appendix \ref{sec:proof-cvx} we show that under the condition in the theorem, $\gamma_{p,q}(W)\leq\gamma$.  To complete the proof, we also show $\gamma^d_{p,q}$ is homogeneous and that this is sufficient for convexity.
\end{sketch}

\section{Per-Unit and Path Regularization}\label{sec:path}

In this Section we will focus on the special case of $q=\infty$,
 i.e.~when we constrain the norm of the incoming weights of each unit
separately.  

Per-unit $\ell_1$-regularization was studied by
\cite{bartlett98,koltchinskii02,bartlett03} who showed generalization
guarantees.  A two-layer network of this form with RELU activation was
also considered by \cite{Bach14}, who studied its
approximation ability and suggested heuristics for learning
it.  Per-unit $\ell_2$ regularization in a two-layer network was
considered by \cite{Cho09}, who showed it is equivalent to using a
specific kernel. We now introduce \emph{Path regularization} and discuss
its equivalence to Per-Unit regularization.

\paragraph{Path Regularization}
Consider a regularizer which looks at the sum over all paths from
input nodes to the output node, of the product of the weights along
the path:
\begin{equation}
  \label{eq:pathr}
  \pathr_p(w) = \Bigl(
    \sum_{\vin[i]\overset{e_1}{\rightarrow}v_1\overset{e_2}{\rightarrow}v_2\cdots\overset{e_k}{\rightarrow}\vout} \prod_{i=1}^k \abs{w(e_i)}^p \Bigr)^{1/p}
\end{equation}
where $p\geq 1$ controls the norm used to aggregate the paths.  We can
motivate this regularizer as follows: if a node does not have any
high-weight paths going out of it, we really don't care much about
what comes into it, as it won't have much effect on the output.  The
path-regularizer thus looks at the aggregated influence of all the
weights.

Referring to the induced regularizer $\pathr^G_p(f) = \min_{f_{G,w}=f}
\pathr_p(w)$ (with the usual shorthands for layered graphs), we now
observe that for layered graphs, path regularization and per-unit
regularization are equivalent:
\begin{theorem}\label{thm:path-layer}
For $p\geq 1$, any $d$ and (finite or infinite) $H$, for any $f\in\calN^{d,H}$:  $\pathr_p^{d,H}(f) = \gamma^{d,H}_{p,\infty}$
\end{theorem}
It is important to emphasize that even for layered graphs, it is not
the case that for all weights $\pathr_p(w)=\gamma_{p,\infty}(w)$.
E.g., a high-magnitude edge going into a unit with no non-zero
outgoing edges will affect $\gamma_{p,\infty}(w)$ but not
$\pathr_p(w)$, as will having high-magnitude edges on different layers
in different paths.  In a sense path regularization is as more
careful regularizer less fooled by imbalance.  Nevertheless, in the
proof of Theorem \ref{thm:path-layer} in Appendix \ref{sec:path-layer}, we show we
can always balance the weights such that the two measures are equal.

The equivalence does not extend to non-layered graphs, since the
lengths of different paths might be different.  Again, we can think of
path regularizer as more refined regularizer taking into account the
local structure.  However, if we consider all DAGs of depth at most
$d$ (i.e.~with paths of length at most $d$), the notions are again
equivalent (see proof in Appendix \ref{sec:proof-path-dag}):
\begin{theorem}\label{thm:path-dag}
  For any $p\geq 1$ and any $d$: $\displaystyle \gamma^d_{p,\infty}(f) =
  \min_{\textrm{$G\in \DAG(d)$}} \pathr^G_p(f)$.
\end{theorem}

In particular, for any graph $G$ of depth $d$, we have that
$\pathr^G_p(f) \geq\gamma^d_{p,\infty}(f)$.  Combining this
observation with Corollary \ref{cor:noH} allows us to immediately obtain a
generalization bound for path regularization on any, even non-layered,
graph:
\begin{corollary}
  For any graph $G$ of depth $d$ and any set
  $S=\{x_1,\dots,x_m\}\subseteq\R^D$:
$$\calR_m(\calN^G_{\pathr_1\leq \pathr}) \leq \sqrt{\frac{4^{d-1} \pathr^2
    \cdot 4\log(2D) \sup \norm{x_i}_\infty^2}{m}}$$
\end{corollary}
Note that in order to apply Corollary \ref{cor:noH} and obtain a
width-independent bound, we had to limit ourselves to $p=1$.  We
further explore this issue next.

\paragraph{Capacity}

As was previously noted, size-independent generalization bounds for
bounded depth networks with bounded per-unit $\ell_1$ norm have long
been known (and make for a popular homework problem).  These
correspond to a specialization of Corollary \ref{cor:noH} for the case
$p=1,q=\infty$.  Furthermore, the kernel view of \cite{Cho09} allows
obtaining size-independent generalization bound for {\em two-layer}
networks with bounded per-unit $\ell_2$ norm (i.e.~a single infinite
hidden layer of all possible unit-norm units, and a bounded
$\ell_2$-norm output unit).  However, the lower bound of Theorem
\ref{thm:shattering} establishes that for any $p>1$, once we go beyond
two layers, we cannot ensure generalization without also controlling
the size (or width) of the network.

\paragraph{Convexity}
An immediately consequence of Theorem \ref{thm:cvx} is that per-unit
regularization, if we do not constrain the network width, is convex
for any $p\geq 1$.  In fact, $\gamma^d_{p,\infty}$ is a (semi)norm.
However, as discussed above, for depth $d>2$ this is meaningful only
for $p=1$, as $\gamma^d_{p,\infty}$ collapses for $p>1$.

\paragraph{Hardness} Since the classes $\calN^d_{1,\infty}$ are
convex, we might hope that this might make learning computationally
easier.  Indeed, one can consider functional-gradient or boosting-type
strategies for learning a predictor in the class \citep{lee96}.  However, as
\citet{Bach14} points out, this is not so easy as it
requires finding the best fit for a target with a RELU unit, which is
not easy.  Indeed, applying results on hardness of learning
intersections of halfspaces, which can be represented with small
per-unit norm using two-layer networks, we can conclude that, subject
to certain complexity assumptions, it is not possible to
efficiently PAC learn $\calN^d_{1,\infty}$, even for depth $d=2$ when $\gamma_{1,\infty}$ increases superlinearly:
\begin{corollary}\label{cor:hard1}
  Subject to the the strong random CSP assumptions in \cite{Daniely14}, it is
  not possible to efficiently PAC learn (even improperly) functions
  $\{\pm 1\}^D \rightarrow \{ \pm 1 \}$ realizable with unit margin by
  $\calN^2_{1,\infty}$ when $\gamma_{1,\infty}=\omega(D)$ (e.g.~when
  $\gamma_{1,\infty}=D \log D$). Moreover, subject to intractability of 
  $\tilde{Q}(D^{1.5})$-unique shortest vector problem, for any $\epsilon>0$,
  it is not possible to efficiently PAC learn (even improperly) functions
  $\{\pm 1\}^D \rightarrow \{ \pm 1 \}$ realizable with unit margin by
  $\calN^2_{1,\infty}$ when $\gamma_{1,\infty}=D^{1+\epsilon}$.
\end{corollary}
This is a corollary of Theorem \ref{thm:hardness} in the Appendix
\ref{sec:hardness}.  Either versions of corollary \ref{cor:hard1} precludes the
possibility of learning in time polynomial in $\gamma_{1,\infty}$,
though it still might be possible to learn in $\textrm{poly}(D)$ time
when $\gamma_{1,\infty}$ is sublinear.

\paragraph{Sharing} We conclude this Section with an observation on
the type of networks obtained by per-unit, or equivalently path,
regularization.
\begin{theorem}\label{thm:opt-tree}
  For any $p\geq 1$ and $d>1$ and any $f\in\calN^d$, there exists a
  layered graph $G(V,E)$ of depth $d$, such that $f\in\calN^G$ and
  $\gamma^G_{p,\infty}(f)=\pathr^G_p(f)=\gamma^d_{p,\infty}(f)$, and
  the out-degree of every internal (non-input) node in $G$ is one.
  That is, the subgraph of $G$ induced by the non-input vertices is a
  tree directed toward the output vertex.
\end{theorem}
What the Theorem tells us is that we can realize every function as a
tree with optimal per-unit norm.  If we think of learning with an
infinite fully-connected layered network, we can always restrict
ourselves to models in which the non-zero-weight edges form a tree.
This means that when using per-unit regularization we have no
incentive to ``share'' lower-level units---each unit will only have a
single outgoing edge and will only be used by a single down-stream
unit.  This seems to defy much of the intuition and power of using
deep networks, where we expect lower layers to represent generic
feature useful in many higher-level features.  In effect, we are not
encouraging any transfer between learning different aspects of the
function (or between different tasks or classes, if we do have
multiple output units).  Per-unit regularization therefore misses out
on much of the inductive bias that we might like to impose when using
deep learning (namely, promoting sharing).

\begin{proof}[of Theorem \ref{thm:opt-tree}]
For any $f_{G,w}\in \calN^{\DAG(d)}$, we show how to construct such $\widetilde{G}$ and $\widetilde{w}$. 
We first sort the vertices of $G$ based on topological ordering such that the out-degree of the first vertex is zero.
Let $G_0=G$ and $w_0=w$. At each step $i$, we first set $G_i=G_{i-1}$ and $w_i=w_{i-1}$ and then pick the vertex $u$ that is the $i$th vector in the topological ordering.
If the out-degree of $u$ is at most 1. Otherwise, for any edge $(u\rightarrow v)$ we create a copy of vertex $u$ that we call it $u_v$, add the edge $(u_v\rightarrow v)$ to $G_i$ and connect all incoming edges of $u$ with the same weights to every such $u_v$ and finally we delete the vertex $u$ from $G_i$ together with all incoming and outgoing edges of $u$. It is easy to indicate that $f_{G_{i},w_i}=f_{G_{i-1},w_{i-1}}$. After at most $|V|$ such steps, all internal nodes have out-degree one and hence the subgraph induced by non-input vertices will be a tree.
\end{proof}

\section{Overall Regularization}
\label{sec:overall}
In this Section, we will focus on ``overall'' $\ell_p$ regularization,
corresponding to the choice $q=p$, i.e.~when we bound the overall
(vectorized) norm of all weights in the system:
$$\mu_{p,p}(w)=\Bigl( \sum_{e\in E} \abs{w(e)}^p \Bigr)^{1/p}.$$

\paragraph{Capacity}
For $p\leq 2$, Corollary \ref{cor:noH} provides a generalization
guarantee that is independence of the width---we can conclude that if
we use weight decay (overall $\ell_2$ regularization), or any tighter
$\ell_p$ regularization, there is no need to limit ourselves to
networks of finite size (as long as the corresponding dual-norm of the
inputs are bounded).  However, in Section \ref{sec:tight} we saw that
with $d \geq 3$ layers, the regularizer degenerates and leads to
infinite capacity classes if $p>2$.  In any case, even if we bound the
overall $\ell_1$-norm, the complexity increases exponentially with the
depth.

\paragraph{Convexity} The conditions of Theorem \ref{thm:cvx} for
convexity of $\calN^d_{2,2}$ are ensured when $p \geq d$.  For depth
$d=1$, i.e.~a single unit, this just confirms that
$\ell_p$-regularized linear prediction is convex for $p\geq 1$.  For
depth $d=2$, we get convexity with $\ell_2$ regularization, but not
$\ell_1$.  For depth $d>2$ we would need $p>d\geq 3$, however for such
values of $p$ we know from Theorem \ref{thm:shattering} that
$\calN^d_{p,p}$ degenerates to an infinite capacity class if we do not
control the width (if we do control the width, we do not get
convexity).  This leaves us with $\calN^2_{2,2}$ as the interesting
convex class.  Below we show an explicit convex characterization of
$\calN^2_{2,2}$ by showing it is equivalent to so-called ``convex neural
nets''.

{\em Convex Neural Nets} \citep{Bengio05} over inputs in $\R^D$ are
two-layer networks with a fixed infinite hidden layer consisting of
all units with weights $w\in\GG$ for some base class $\GG\in\R^D$, and
a second $\ell_1$-regularized layer.  Since over finite data the
weights in the second layer can always be taken to have finite support
(i.e.~be non-zero for only a finite number of first-layer units), and
we can approach any function with countable support, we
can instead think of a network in $\calN^2$ where the bottom layer is
constraint to $\GG$ and the top layer is $\ell_1$ regularized.
Focusing on $\GG=\{ w \,|\, \norm{w}_p \leq 1 \}$, this corresponds to
imposing an $\ell_p$ constraint on the bottom layer, and $\ell_1$
regularization on the top layer and yields the following complexity
measure over $\calN^2$:
\begin{equation}
  \label{eq:convexNN}
  \convexnn_p(f) = \inf_{f_{\layer(d),W}=f, \textrm{s.t.} \forall_j
    \norm{W_1[j,:]}_p \leq 1} \norm{W_2}_1.
\end{equation}
This is similar to per-unit regularization, except we impose different
norms at different layers (if $p\not=1$).  We can see that
$\calN^2_{\convexnn_p \leq \convexnn} = \convexnn \cdot
\overline{\conv}(\sigma(\GG))$, and is thus convex for any $p$.
Focusing on RELU activation we have the equivalence:
\begin{theorem}
  $\displaystyle \mu^2_{2,2}(f) = 2 \convexnn_2(f).$
\end{theorem}
That is, overall $\ell_2$ regularization with two layers is equivalent
to a convex neural net with $\ell_2$-constrained units on the bottom
layer and $\ell_1$ (not $\ell_2$!) regularization on the output.
\begin{proof}We can calculate:
\begin{align}
\min_{f_W=f}
\mu_{2,2}^2(W)&=\min_{f_W=f}\sum_{j=1}^{H}\left(\sum_{i=1}^{D}|W_1[j,i]|^2+|W_2[j]|^2\right)
\notag \\
&=\min_{f_W=f}
\sum_{j=1}^{H}2\sqrt{\sum\nolimits_{i=1}^{D}|W_1[j,i]|^2}\cdot|W_2[j]|
\label{eq:convexnnproof1} \\
&=2 \min_{f_W=f} \sum_{j=1}^{H}\left|W_2[j]\right|\quad \text{s.t.}\quad 
\sqrt{\sum\nolimits_{i=1}^{D}|W_1[j,i]|^2}\leq 1. \label{eq:convexnnproof2}
\end{align}
Here \eqref{eq:convexnnproof1} is the arithmetic-geometric mean
inequality for which we can achieve equality by balancing the weights
(as in Claim \ref{clm:mugamma}) and \eqref{eq:convexnnproof2} again
follows from the homogeneity of the RELU which allows us to rebalance
the weights.
\end{proof}

\paragraph{Hardness} As with $\calN^d_{1,\infty}$, we might hope that the convexity of
$\calN^2_{2,2}$ might make it computationally easy to learn.  However,
by the same reduction from learning intersection of halfspaces
(Theorem \ref{thm:hardness} in Appendix \ref{sec:hardness}) we can
again conclude that we cannot learn in time polynomial in $\mu^2_{2,2}$:

\begin{corollary}\label{cor:hard2}
  Subject to the the strong random CSP assumptions in \cite{Daniely14}, it is
  not possible to efficiently PAC learn (even improperly) functions
  $\{\pm 1\}^D \rightarrow \{ \pm 1 \}$ realizable with unit margin by
  $\calN^2_{p,p}$ when $\mu^2_{p,p}=\omega(D^{\frac{1}{p}})$. (e.g.~when
  $\gamma_{1,\infty}=D \log D$). Moreover, subject to intractability of 
  $\tilde{Q}(D^{1.5})$-unique shortest vector problem, for any $\epsilon>0$,
  it is not possible to efficiently PAC learn (even improperly) functions
  $\{\pm 1\}^D \rightarrow \{ \pm 1 \}$ realizable with unit margin by
  $\calN^2_{1,\infty}$ when $\gamma_{1,\infty}=D^{\frac{1}{p}+\epsilon}$.
\end{corollary}


\section{Depth Independent Regularization}
\label{sec:nod}
Up until now we discussed relying on magnitude-based regularization
instead of directly controlling network size, thus allowing unbounded
and even infinite width. But we still relied on a finite bound on the
depth in all our derivations.  Can the explicit dependence on the
depth be avoided, and replaced with only a measure of scale of the
weights?

We already know we cannot rely only on a bound on the group-norm
$\mu_{p,q}$ when the depth is unbounded, as we know from Theorem
\ref{thm:shattering} that in terms of $\mu_{p,q}$ the sample
complexity necessarily increases exponentially with the depth: if we
allow arbitrarily deep graphs we can shrink $\mu_{p,q}$ toward zero
without changing the scale of the computed function.  However,
controlling the $\gamma$-measure, or equivalently the path-regularizer
$\pathr$, in arbitrarily-deep graphs is sensible, and we can define:
\begin{equation}
  \label{eq:infgamma}
  \gamma_{p,q} = \inf_{d\geq 1} \gamma^d_{p,q}(f) =
  \lim_{d\rightarrow\infty} \gamma^d_{p,q}(f) \quad\quad\text{or:}\quad
  \pathr_p = \inf_G \pathr^G_p(f)
\end{equation}
where the minimization is over {\em any} DAG.  From Theorem \ref{thm:path-dag}
we can conclude that $\pathr_p(f)=\gamma_{p,\infty}(f)$.  In any case,
$\gamma_{p,q}(f)$ is a sensible complexity measure, that does not
collapse despite the unbounded depth.  Can we obtain generalization
guarantees for the class $\calN_{\gamma_{p,q}\leq\gamma}$ ?

Unfortunately, even when $1/p+1/q \geq 1$ and we can obtain
width-independent bounds, the bound in Corollary \ref{cor:noH} still
has a dependence on $4^d$, even if $\gamma_{p,q}$ is bounded.  Can
such a dependence be avoided?

For {\em anti-symmetric} Lipschitz-continuous activation functions
(i.e.~such that $\sigma(-z)=-\sigma(z)$), such as the ramp, and for
per-unit $\ell_p$-regularization $\mu^d_{1,\infty}$ we can avoid the factor of $4^d$
\begin{theorem}\label{thm:antisym}
For any anti-symmetric 1-Lipschitz function $\sigma$ and any set $S=\{x_1,\dots,x_m\}\subseteq\R^D$:
$$
\calR_m(\calN_{\mu_{1,\infty}\leq \mu}^{d,\sigma}) \leq 
\sqrt{\frac{4\mu^{2d} \log(2D) \sup \norm{x_i}_{\infty}^2}{m}}
$$
\end{theorem}
The proof is again based on an inductive argument similar to Theorem \ref{thm:l-norm} and you can find it in appendix \ref{sec:antisym}.

However, the ramp is not homogeneous and so the equivalent between
$\mu$, $\gamma$ and $\phi$ breaks down.  Can we obtain such a bound
also for the RELU?  At the very least, what we can say is that an
inductive argument such that used in the proofs of Theorems
\ref{thm:l-norm} and \ref{thm:antisym} cannot be used to avoid an
exponential dependence on the depth.  To see this, consider
$\gamma_{1,\infty}\leq 1$ (this choice is arbitrary if we are
considering the Rademacher complexity), for which we have
\begin{equation}\label{eq:Nd1}
\calN^{d+1}_{\gamma_{1,\infty}<1} = \left[
  \overline{\conv}(\calN^d_{\gamma_{1,\infty}<1}) \right]_+,
\end{equation}
where $\overline\conv(\cdot)$ is the symmetric convex hull, and
$[\cdot]_+ = \max(z,0)$ is applied to each function in the class.  In
order to apply the inductive argument without increasing the
complexity exponentially with the depth, we would need the operation
$[ \overline{\conv}(\HH) ]_+$ to preserve the Rademacher complexity,
at least for non-negative convex cones $\HH$.  However we show a
simple example of a non-negative convex cone $\HH$ for which
$\calR_m\left( [ \overline{\conv}(\HH) ]_+ \right) > \calR_m\left(
  \HH \right)$.

We will specify $\HH$ as a set of vectors in $\R^m$, corresponding
to the evaluation of $h(x_i)$ of different functions in the class on
the $m$ points $x_i$ in the sample.  In our construction, we will have
only $m=3$ points.  Consider $\HH = \conv(\{ (1,0,1),(0,1,1)
\})$, in which case $\HH' \defeq [ \overline{\conv}(\HH) ]_+ =
\conv(\{ (1,0,1),(0,1,1),(0.5,0,0) \})$.  It is not hard to verify
that $\calR_m(\HH')=\frac{13}{16}>\frac{12}{16}=\calR_m(\HH)$.

\section{Summary and Open Issues}

We presented a general framework for norm-based capacity control for
feed-forward networks, and analyzed when the norm-based control is
sufficient and to what extent capacity still depends on other
parameters.  In particular, we showed that in depth $d>2$ networks,
per-unit control with $p>1$ and overall regularization with $p>2$ is
not sufficient for capacity control without also controlling the
network size.  This is in contrast with linear models, where with any
$p<\infty$ we have only a weak dependence on dimensionality, and
two-layer networks where per-unit $p=2$ is also sufficient for
capacity control.  We also obtained generalization guarantees for
perhaps the most natural form of regularization, namely $\ell_2$
regularization, and showed that even with such control we still
necessarily have an exponential dependence on the depth.  

Although the additive $\mu$-measure and multiplication
$\gamma$-measure are equivalent at the optimum, they behave rather
differently in terms of optimization dynamics (based on anecdotal
empirical experience) and understanding the relationship between them,
as well as the novel path-based regularizer can be helpful in practical
regularization of neural networks.  

Although we obtained a tight characterization of when size-independent
capacity control is possible, the precise polynomial dependence of
margin-based classification (and other tasks) on the norm in might not
be tight and can likely be improved, though this would require going
beyond bounding the Rademacher complexity of the real-valued class.
In particular, Theorem \ref{thm:l-norm} gives the same bound for
per-unit $\ell_1$ regularization and overall $\ell_1$ regularization,
although we would expect the later to have lower capacity.

Beyond the open issue regarding depth-independent $\gamma$-based
capacity control, another interesting open question is understanding
the expressive power of $\calN^d_{\gamma_{p,q}\leq\gamma}$,
particularly as a function of the depth $d$.  Clearly going from depth
$d=1$ to depth $d=2$ provides additional expressive power, but it is
not clear how much additional depth helps.  The class $\calN^2$
already includes all binary functions over $\{\pm 1\}^D$ and is dense
among continuous real-valued functions.  But can the $\gamma$-measure
be reduced by increasing the depth?  Viewed differently:
$\gamma^d_{p,q}(f)$ is monotonically non-increasing in $d$, but are
there functions for it continues decreasing?  Although it seems
obvious there are functions that require high depth for efficient
representation, these questions are related to decade-old problems in
circuit complexity and might not be easy to resolve.

\acks{This research was partially supported by NSF grant IIS-1302662 and an Intel ICRI-CI award.
We thank the COLT anonymous reviewers for pointing out an error in the
statement of Lemma~\ref{lem:layer1} and suggesting other corrections.}

\bibliography{biblio}

\appendix

\section{Rademacher Complexities}\label{sec:rademacher}
The sample based Rademacher complexity of a class $\calF$ of function mapping from $\calX$ to $\R$ with respect to a set $S=\{x_1,\dots,x_m\}$ is defined as:
$$
\calR_m(\calF) = \E_{\xi \in \{\pm 1\}^m}\left[\frac{1}{m}
\sup_{f\in \calF}  \left\lvert \sum_{i=1}^m \xi_i f(x_i) \right\rvert \right]
$$

In this section, we prove an upper bound for the Rademacher complexity
of the class $\calN_{\gamma_{p,q}\leq
\gamma}^{d,H,\sigma_{\text{RELU}}}$, i.e., the class of functions that
can be represented as depth $d$, width $H$ network with rectified linear
activations, and
the layer-wise group norm complexity $\gamma_{p,q}$ bounded by
$\gamma$. As mentioned in the main text, our proof is an induction with
respect to the depth $d$.
We start with $d=1$ layer neural networks, which is essentially  the class of linear separators.
\subsection{$\ell_p$-regularized Linear Predictors}\label{sec:linear}

For completeness, we prove the upper bounds on the Rademacher complexity of class of linear separators with bounded $\ell_p$ norm. The upper bounds presented here are particularly similar to generalization bounds in \cite{kakade09} and \cite{balcan14}. We first mention two already established lemmas that we use in the proofs.
\begin{theorem}(Khintchine-Kahane Inequality)
For any $0<p<\infty$ and $S=\{z_1,\dots,z_m\}$, if the random 
variable $\xi$ is uniform over $\{\pm 1\}^m$, then
$$
\left(\E_{\xi}\left[\left\lvert \sum_{i=1}^m \xi_i z_i
\right\rvert ^p\right]\right)^\frac{1}{p} \leq C_p 
\left(\sum_{i=1}^m |z_i|^2\right)^\frac{1}{2}
$$
where $C_p$ is a constant depending only on $p$.
\end{theorem}
The sharp value of the constant $C_p$ was found by 
\citet{haagerup81} but for our analysis, it is enough to note
 that if $p\geq 1$ we have $C_p \leq \sqrt{p}$.
\begin{lemma}(Massart Lemma)
Let $A$ be a finite set of $m$ dimensional vectors. Then
$$
\E_{\xi}\left[
\max_{a\in A}\frac{1}{m}\sum_{i=1}^{m}\xi_ia_i
\right] \leq \max_{a\in A} \norm{a}_2 \frac{\sqrt{2\log |A|}}{m},
$$
where $|A|$ is the cardinality of $A$.
\end{lemma}

We are now ready to show upper bounds on Rademacher complexity of linear separators with bounded $\ell_p$ norm.

\begin{lemma}\label{lem:layer1}(Rademacher complexity of linear separators with bounded $\ell_p$ norm)
For any $d,q\geq 1$,
For any $1\leq p\leq 2$,
$$
\calR_m
(\calN^1_{\gamma_{p,q}\leq \gamma}) \leq \sqrt{ \frac{ \gamma^2\min\{p^*,4\log(2D)\} \max_i \norm{x_i}_{p^*}^2}{m}}
$$
and for any $2<p<\infty$
$$
\calR_m
(\calN^1_{\gamma_{p,q}\leq \gamma})\leq \frac{\sqrt{2}\gamma\norm{X}_{2,p^*}}{m} \leq \frac{\sqrt{2}\gamma\max_i\norm{x_i}_{p^*}}{m^{\frac{1}{p}}}
$$
where $p^*$ is such that $\frac{1}{p^*} + \frac{1}{p}=1$.
\end{lemma}
\begin{proof}
First, note that $\calN^1$ is the class of linear functions
and hence for any function $f_{w}\in\calN^1$, we have
that $\gamma_{p,q}(w)=\norm{w}_p$. Therefore,
we can write the Rademacher complexity for a set 
$S=\{x_1,\dots,x_m\}$ as:
\begin{align*}
\calR_m(\calN^1_{\gamma_{p,q}\leq \gamma})
&= \E_{\xi \in \{\pm 1\}^m}\left[\frac{1}{m}\sup_{\norm{w}_p\leq \gamma}
  \left\lvert \sum_{i=1}^m \xi_i w^\top x_i \right\rvert \right]\\
&= \E_{\xi \in \{\pm 1\}^m}\left[\frac{1}{m}\sup_{\norm{w}_p\leq \gamma}
  \left\lvert w^\top\sum_{i=1}^m \xi_i x_i \right\rvert \right]\\
&= \gamma\E_{\xi \in \{\pm 1\}^m}\left[\frac{1}{m}
\norm{ \sum_{i=1}^m \xi_i x_i}_{p^*}\right]
\end{align*}
For $1\leq p\leq \min\left\{2,\frac{2\log(2D)}{2\log(2D)-1}\right\}$ (and therefore $2\log(2D) \leq p^*$), we have
\begin{align*}
\calR_m(\calN^1_{\gamma_{p,q}\leq \gamma}) &=\gamma\E_{\xi \in \{\pm 1\}^m}
\left[\frac{1}{m}\norm{ \sum_{i=1}^m \xi_i x_i}_{p^*}\right]\\
&\leq D^{\frac{1}{p^*}}\gamma\E_{\xi \in \{\pm 1\}^m}
\left[\frac{1}{m}\norm{ \sum_{i=1}^m \xi_i x_i}_{\infty}\right]\\
&\leq D^{\frac{1}{2\log(2D)}}\gamma\E_{\xi \in \{\pm 1\}^m}
\left[\frac{1}{m}\norm{ \sum_{i=1}^m \xi_i x_i}_{\infty}\right]\\
&\leq \sqrt{2}\gamma\E_{\xi \in \{\pm 1\}^m}
\left[\frac{1}{m}\norm{ \sum_{i=1}^m \xi_i x_i}_{\infty}\right]\\
\end{align*}
We now use  the Massart Lemma viewing each feature $(x_i[j])_{i=1}^{m}$ for
 $j=1,\ldots,D$ as a  member of a finite hypothesis class and obtain
 \begin{align*}
\calR_m(\calN^1_{\gamma_{p,q}\leq \gamma}) &\leq  \sqrt{2}\gamma\E_{\xi \in \{\pm 1\}^m}
\left[\frac{1}{m}\norm{ \sum_{i=1}^m \xi_i x_i}_{\infty}\right]\\
&\leq 2\gamma\frac{\sqrt{\log(2D)}}{m}\max_{j=1\ldots,D}\norm{(x_i[j])_{i=1}^{m}}_2\\
&\leq 2\gamma \sqrt{\frac{\log(2D)}{m}}\max_{i=1,\ldots,m}\norm{x_i}_{\infty}\\
&\leq 2\gamma \sqrt{\frac{\log(2D)}{m}}\max_{i=1,\ldots,m}\norm{x_i}_{p^*}
\end{align*}
If $\min\left\{2,\frac{2\log(2D)}{2\log(2D)-1}\right\} <p<\infty$, by Khintchine-Kahane inequality we have
\begin{align*}
\calR_m(\calN^1_{\gamma_{p,q}\leq \gamma}) &=\gamma\E_{\xi \in \{\pm 1\}^m}
\left[\frac{1}{m}\norm{ \sum_{i=1}^m \xi_i x_i}_{p^*}\right]\\
&\leq \gamma\frac{1}{m}\left(\sum_{j=1}^D \E_{\xi \in \{\pm 1\}^m}\left[
\left\lvert \sum_{i=1}^m \xi_i x_i[j]\right\rvert ^{p^*}\right]\right)^{1/p^*}\\
&\leq\gamma\frac{\sqrt{p^*}}{m}\left(\sum\nolimits_{j=1}^{D}\norm{(x_i[j])_{i=1}^{m}}_2^{p^*}\right)^{1/p^*} =\gamma\frac{\sqrt{p^*}}{m}\norm{X}_{2,p^*}
\end{align*}
If $p^*\geq 2$, by Minskowski inequality we have that $\norm{X}_{2,p^*} \leq m^{1/2} \max_i \norm{x_i}_{p^*}$. Otherwise, by subadditivity of the function $f(z)=z^{\frac{p^*}{2}}$, we get $\norm{X}_{2,p^*} \leq m^{1/{p^*}} \max_i \norm{x_i}_{p^*}$.

\end{proof}

\subsection{Theorem~\ref{thm:l-norm}}
We define the hypothesis class $\calN^{d,H,H}$ to be the class of functions from $\calX$ to $\R^H$ computed by a layered network of depth $d$, layer size $H$ and $H$ outputs.

For the proof of theorem~\ref{thm:l-norm}, we need the following two
technical lemmas.
The first is the well-known contraction lemma:
\begin{lemma}(Contraction Lemma)
Let function $\phi:\R\rightarrow \R$ be Lipschitz with constant $\calL_\phi$ such that $\phi$ satisfies $\phi(0)=0$. Then for any class $\calF$ of functions mapping from $\calX$ to $\R$ and any set $S=\{x_1,\dots,x_m\}$:
$$
\E_{\xi \in \{\pm 1\}^m}\left[\frac{1}{m}
\sup_{f\in \calF}  \left\lvert \sum_{i=1}^m \xi_i \phi( f(x_i) ) \right\rvert \right] \leq 2 \calL_\phi \E_{\xi \in \{\pm 1\}^m}\left[\frac{1}{m}
\sup_{f\in \calF}  \left\lvert \sum_{i=1}^m \xi_i f(x_i) ) \right\rvert \right]
$$
\end{lemma}
Next, the following lemma reduces the maximization over a matrix
$W\in\R^{H\times H}$ that
appears in the computation of Rademacher complexity to $H$ independent
maximizations over a vector $w\in\R^{H}$ (the proof is deferred to
Subsection \ref{sec:proof-lem-l-norm}):
\begin{lemma}\label{lem:l-norm}
For any $p,q\geq 1$, $d\geq 2$, $\xi \in \{\pm 1\}^m$ and
 $f\in \calN^{d,H,H}$ we have
$$
\sup_{W}\frac{1}{\norm{W}_{p,q}}
\norm{\sum_{i=1}^m \xi_i [W[f(x_i)]_+]_+}_{p^*}=H^{[\frac{1}{p^*} - \frac{1}{q}]_+}
 \sup_{w} \frac{1}{\norm{w}_{p}} \left\lvert \sum_{i=1}^m \xi_i 
 [w^\top [f(x_i)]_+]_+\right\rvert 
$$
where $p^*$ is such that $\frac{1}{p^*} + \frac{1}{p}=1$.
\end{lemma}

\begin{reptheorem}{thm:l-norm}
For any $d,p,q\geq 1$ and any set $S=\{x_1,\dots,x_m\}\subseteq\R^D$:
$$
\calR_m(\calN_{\gamma_{p,q}\leq \gamma}^{d,H,\relu}) \leq 
\sqrt{\frac{\gamma^2 \left( 2 H^{[\frac{1}{p^*} -
        \frac{1}{q}]_+}\right)^{2(d-1)} \min\{p^*,2\log(2D)\} \sup \norm{x_i}_{p^*}^2}{m}}
$$
and so:
$$
\calR_m(\calN_{\mu_{p,q}\leq \mu}^{d,H,\relu}) \leq 
\sqrt{\frac{\mu^{2d} \left( 2 H^{[\frac{1}{p^*} -
        \frac{1}{q}]_+} /\sqrt[q]{d}\right)^{2(d-1)} \min\{p^*,2\log(2D)\} \sup \norm{x_i}_{p^*}^2}{m}}
$$
where $p^*$ is such that $\frac{1}{p^*} + \frac{1}{p}=1$.
\end{reptheorem}
\begin{proof}
By the definition of Rademacher complexity if $\xi$ is uniform over $\{\pm 1\}^m$, we have:
\begin{align}\notag
\calR_m(\calN^{d,H}_{\gamma_{p,q}\leq \gamma}) &= \E_{\xi}\left[\frac{1}{m} \sup_{f\in \calN^{d,H}_{\gamma_{p,q}\leq \gamma}} \left\lvert \sum_{i=1}^m 
\xi_i f(x_i) \right\rvert \right]\\ \notag
&= \E_{\xi}\left[\frac{1}{m} 
\sup_{f\in\calN^{d,H}}  \frac{\gamma}{\gamma_{p,q}(f)}
\left\lvert \sum_{i=1}^m \xi_i f(x_i) \right\rvert \right]\\ \notag
&= \E_{\xi}\left[\frac{1}{m}  
\sup_{g\in \calN^{d-1,H,H}}  \sup_{w} 
\frac{\gamma}{\gamma_{p,q}(g)\norm{w}_p} \left\lvert \sum_{i=1}^m \xi_i 
w^\top [g(x_i)]_+ \right\rvert \right]\\ \notag
&=\E_{\xi }\left[\frac{1}{m}  
\sup_{g\in \calN^{d-1,H,H}}   \frac{\gamma}{\gamma_{p,q}(g)} 
\norm{\sum_{i=1}^m \xi_i [g(x_i)]_+}_{p^*}\right]\\ \notag
&=\E_{\xi }\left[\frac{1}{m}  
\sup_{h\in \calN^{d-2,H,H}}  \frac{\gamma}{\gamma_{p,q}(h)}\sup_{W}\frac{1}{\norm{W}_{p,q}}
\norm{\sum_{i=1}^m \xi_i [W[h(x_i)]_+]_+}_{p^*}\right]\\ \label{eq:lemlayer}
&= H^{[\frac{1}{p^*} - \frac{1}{q}]_+}\E_{\xi}
\left[\frac{1}{m}  \sup_{h\in \calN^{d-2,H,H}} \frac{\gamma}{\gamma_{p,q}(h)} \sup_{w} 
\frac{1}{\norm{w}_{p}} \left\lvert \sum_{i=1}^m \xi_i [w^\top [h(x_i)]_+]_+
\right\rvert \right]\\\notag
&= H^{[\frac{1}{p^*} - \frac{1}{q}]_+}\E_{\xi }
\left[ \frac{1}{m} \sup_{g\in \calN^{d-1,H}_{\gamma_{p,q}\leq \gamma}} \left\lvert  \sum_{i=1}^m \xi_i [g(x_i)]_+\right\rvert  \right]\\ \label{eq:contraction}
&\leq 2 H^{[\frac{1}{p^*} - \frac{1}{q}]_+}\E_{\xi}
\left[ \frac{1}{m} \sup_{g\in \calN^{d-1,H}_{\gamma_{p,q}\leq \gamma}} \left\lvert  \sum_{i=1}^m \xi_i g(x_i)\right\rvert  \right]\\ \notag
&=2 H^{[\frac{1}{p^*} - \frac{1}{q}]_+} \calR_m(\calN^{d-1,H}_{\gamma_{p,q}\leq \gamma}) \\ \notag
\end{align}
where the equality~\eqref{eq:lemlayer} is obtained by lemma~\ref{lem:l-norm} and inequality~\eqref{eq:contraction} is by Contraction Lemma.
This will give us the bound on Rademacher complexity of $\calN^{d,H}_{\gamma_{p,q}\leq \gamma}$ based on the Rademacher complexity of  $\calN^{d-1,H}_{\gamma_{p,q}\leq \gamma}$. Applying the same argument on  all layers and using lemma~\ref{lem:layer1} to bound the  complexity of the first layer completes the proof.
\end{proof}

\subsection{Proof of Lemma~\ref{lem:l-norm}}
\label{sec:proof-lem-l-norm}
\begin{proof}
It is immediate that the right hand side of the equality in the
 statement is always less than or equal to the left hand side because
 given any vector $w$ in the right hand side, by setting each row of
 matrix $W$ in the left hand side we get the equality. Therefore, it is
 enough to prove that the left hand side is less than or equal to the
 right hand side. For the convenience of notations, let $g(w) \defeq |\sum_{i=1}^m \xi_i 
w^\top [f(x_i)]_+|$. Define $\widetilde{w}$ to be:
$$
\widetilde{w} \defeq \arg\max_{w} \frac{g(w)}{\norm{w}_{p}}
$$
If $q\leq p^*$, then the right hand side of equality in the lemma statement will reduce to $g(\widetilde{w})/\norm{\widetilde{w}}_p$ and therefore we need to show that for any matrix $V$, 
$$
\frac{g(\widetilde{w})}{\norm{\widetilde{w}}_p} \geq 
\frac{\norm{g(V)}_{p^*}}{\norm{V}_{p,q}}.
$$
Since $q\leq p^*$, we have $\norm{V}_{p,{p^*}} \leq 
 \norm{V}_{p,q}$ and hence it is enough to prove the
  following inequality:
$$
\frac{g(\widetilde{w})}{\norm{\widetilde{w}}_p} \geq 
\frac{\norm{g(V)}_{p^*}}{\norm{V}_{p,{p^*}}}.
$$
On the other hand, if $q>{p^*}$, then we need to prove the following 
inequality holds:
$$
H^{\frac{1}{{p^*}} - \frac{1}{q}}\frac{g(\widetilde{w})}
{\norm{\widetilde{w}}_p} \geq \frac{\norm{g(V)}_{p^*}}{\norm{V}_{p,q}}
$$
Since $q>{p^*}$, we have that $\norm{V}_{p,{p^*}} \leq  
H^{\frac{1}{{p^*}} - \frac{1}{q}}\norm{V}_{p,q}$. Therefore,
 it is again enough to show that:
$$
\frac{g(\widetilde{w})}{\norm{\widetilde{w}}_p} \geq 
\frac{\norm{g(V)}_{p^*}}{\norm{V}_{p,{p^*}}}.
$$
We can rewrite the above inequality in the following form:
$$
\sum_{i=1}^H \left(\frac{g(\widetilde{w})\norm{V_i}_p}
{\norm{\widetilde{w}}_p}\right)^{p^*} \geq \sum_{i=1}^H g(V_i)^{p^*}
$$
By the definition of $\widetilde{w}$, we know that the
 above inequality holds for each term in the sum and 
 hence the inequality is true.
\end{proof}

\subsection{Theorem \ref{thm:antisym}}\label{sec:antisym}

The proof is similar to the proof of theorem \ref{thm:l-norm} but here bounding $\mu_{1,\infty}$ by $\mu$ means the $\ell_1$ norm of input weights to each neuron is bounded by $\mu$. We use a different version of Contraction Lemma in the proof that is without the absolute value:
\begin{lemma}(Contraction Lemma (without the absolute value))
Let function $\phi:\R\rightarrow \R$ be Lipschitz with constant $\calL_\phi$. Then for any class $\calF$ of functions mapping from $\calX$ to $\R$ and any set $S=\{x_1,\dots,x_m\}$:
$$
\E_{\xi \in \{\pm 1\}^m}\left[\frac{1}{m}
\sup_{f\in \calF}\sum_{i=1}^m \xi_i \phi( f(x_i) )\right] \leq \calL_\phi \E_{\xi \in \{\pm 1\}^m}\left[\frac{1}{m}
\sup_{f\in \calF} \sum_{i=1}^m \xi_i f(x_i) )  \right]
$$
\end{lemma}
\begin{reptheorem}{thm:antisym}
For any anti-symmetric 1-Lipschitz function $\sigma$ and any set $S=\{x_1,\dots,x_m\}\subseteq\R^D$:
$$
\calR_m(\calN_{\mu_{1,\infty}\leq \mu}^{d,\sigma}) \leq 
\sqrt{\frac{2\mu^{2d} \log(2D) \sup \norm{x_i}_{\infty}^2}{m}}
$$
\end{reptheorem}
\begin{proof}
Assuming $\xi$ is uniform over $\{\pm 1\}^m$, we have:
\begin{align}\notag
\calR_m(\calN^{d,H}_{\mu_{1,\infty}\leq \mu}) &= \E_{\xi}\left[\frac{1}{m} \sup_{f\in \calN^{d,H}_{\mu_{1,\infty}\leq \mu}} \left\lvert \sum_{i=1}^m 
\xi_i f(x_i) \right\rvert \right]\\ \notag
&= \E_{\xi}\left[\frac{1}{m} \sup_{f\in \calN^{d,H}_{\mu_{1,\infty}\leq \mu}} \sum_{i=1}^m 
\xi_i f(x_i) \right]\\ \notag
&= \E_{\xi}\left[\frac{1}{m}  
\sup_{g\in \calN^{d-1,H,H}_{\mu_{1,\infty}\leq \mu}}  \sup_{\norm{w}_1\leq \mu} 
w^\top\sum_{i=1}^m \xi_i 
\sigma(g(x_i))\right]\\ \notag
&=\E_{\xi }\left[\frac{1}{m}  
\sup_{g\in \calN^{d-1,H,H}_{\mu_{1,\infty}\leq \mu}} 
\norm{\sum_{i=1}^m \xi_i \sigma(g(x_i))}_{\infty}\right]\\ \label{eq:anti-sym}
&= \E_{\xi }
\left[ \frac{1}{m} \sup_{g\in \calN^{d-1,H}_{\mu_{1,\infty}\leq \mu}} \left\lvert  \sum_{i=1}^m \xi_i \sigma(g(x_i))\right\rvert  \right]\\\notag
&= \E_{\xi }
\left[ \frac{1}{m} \sup_{g\in \calN^{d-1,H}_{\mu_{1,\infty}\leq \mu}} \sum_{i=1}^m \xi_i \sigma(g(x_i)) \right]\\ \label{eq:anti-contraction}
&\leq \E_{\xi }
\left[ \frac{1}{m} \sup_{g\in \calN^{d-1,H}_{\mu_{1,\infty}\leq \mu}} \sum_{i=1}^m \xi_i g(x_i) \right]\\\notag
&=\E_{\xi }
\left[ \frac{1}{m} \sup_{g\in \calN^{d-1,H}_{\mu_{1,\infty}\leq \mu}} \left\lvert \sum_{i=1}^m \xi_i g(x_i) \right\rvert  \right]\\\notag
&=\calR_m(\calN^{d-1,H}_{\mu_{1,\infty}\leq \mu}) \\ \notag
\end{align}
where the equality \eqref{eq:anti-sym} is by anti-symmetric property of $\sigma$ and inequality~\eqref{eq:anti-contraction} is by the version of Contraction Lemma without the absolute value. This will give us the bound on Rademacher complexity of $\calN^{d,H}_{\mu_{1,\infty}\leq \mu}$ based on the Rademacher complexity of  $\calN^{d-1,H}_{\mu_{1,\infty}\leq \mu}$. Applying the same argument on  all layers and using lemma~\ref{lem:layer1} to bound the  complexity of the first layer completes the proof.
\end{proof}

\section{Proof that $\gamma^d_{p,q}(f)$ is a semi-norm in $\calN^d$}
\label{sec:proof-cvx}
We repeat the statement here for convenience.
\newtheorem*{thm:cvx}{Theorem \ref{thm:cvx}}
\begin{thm:cvx}
  For any $d,p,q \geq 1$ such that $\frac{1}{q}\leq
  \frac{1}{d-1}\big(1-\frac{1}{p}\big)$, $\gamma^d_{p,q}(f)$ is a semi-norm
  in $\calN^d$.
\end{thm:cvx}
\begin{proof}
The proof consists of three parts. First we show
 that the level set $\calN^d_{\gamma_{p,q}^d\leq
 \gamma}=\{f\in \calN^d: \gamma_{p,q}^d(f)\leq \gamma\}$ is a convex set
 if the condition on $d,p,q$ is satisfied. Next, we
 establish the non-negative homogeneity of $\gamma_{p,q}^d(f)$. Finally,
 we show that if a function $\alpha:\calN^d\rightarrow\R$  is non-negative
 homogeneous and every sublevel set $\{f\in \calN^d : \alpha(f)\leq
 \gamma\}$ is convex, then $\alpha$ satisfies the triangular inequality.

\paragraph{Convexity of the level sets}
First we show that for any two functions $f_1,f_2\in
 \calN^d_{\gamma_{p,q}\leq  \gamma}$ and  $0\leq \alpha \leq 1$, the
 function $g=\alpha f_1 + (1-\alpha)f_2$ is in the hypothesis class
 $\calN^{d}_{\gamma_{p,q}\leq \gamma}$. We prove this by constructing
 weights $W$ that realizes $g$. Let $U$ and $V$ be the weights of two
 neural networks such that $\gamma_{p,q}(U) = \gamma_{p,q}^{d}(f_1)\leq
 \gamma$ and $\gamma_{p,q}(V)=\gamma_{p,q}^d(f_2) \leq \gamma$. 
 For every layer $i=1,\ldots,d$ let
  \begin{equation*}
    \label{eq:W1}
    \tilde{U}_i = \sqrt[d]{\gamma_{p,q}(U)}U_i/{\norm{U_i}_{p,q}},
  \quad\quad \tilde{V}_i = \sqrt[d]{\gamma_{p,q}(V)}V_i/{\norm{V_i}_{p,q}}.
  \end{equation*}
  and set $W_1 = \begin{bmatrix} \tilde{U}_1 \\
    \tilde{V}_1 \end{bmatrix}$ for the first layer, $W_i =
  \begin{bmatrix} \tilde{U}_i & 0 \\ 0 & \tilde{V}_i\end{bmatrix}$ for
  the intermediate layers and $W_d = 
\begin{bmatrix}
\alpha\tilde{U}_d &  (1-\alpha)\tilde{V}_d
\end{bmatrix}$ for the output layer.

Then for the defined $W$, we have $f_W=\alpha f_1 + (1-\alpha)f_2$ for
rectified linear and any other non-negative homogeneous activation
function. Moreover, for any $i<d$, the norm of each layer is
\begin{equation}\label{eq:conv-i}
\norm{W_i}_{p,q} = \left(\gamma_{p,q}(U)^{\frac{q}{d}} + \gamma_{p,q}(V)^{\frac{q}{d}}\right)^{\frac{1}{q}} \leq 2^\frac{1}{q}\gamma^{\frac{1}{d}}
\end{equation}
and in layer $d$ we have:
\begin{equation}\label{eq:conv-d}
\norm{W_d}_p = \left(\alpha^p\gamma_{p,q}(U)^{\frac{p}{d}} + (1-\alpha)^p\gamma_{p,q}(V)^{\frac{p}{d}}\right)^{\frac{1}{p}} \leq 2^{1/p-1}\gamma^{1/d}
\end{equation}
Combining inequalities \eqref{eq:conv-i} and \eqref{eq:conv-d}, we get
$
\gamma^{d}_{p,q}(f_W) \leq 2^{\frac{d-1}{q} + \frac{1}{p}} \gamma \leq \gamma,
$
where the last inequality holds because we assume that
$\frac{1}{q}\leq \frac{1}{d-1}\big(1-\frac{1}{p}\big)$. Thus for every
$\gamma\geq 0$, $\calN^d_{\gamma_{p,q}\leq \gamma}$ is a convex set.

\paragraph{Non-negative homogeneity}
  For any function $f\in\calN^d$ and any $\alpha \geq 0$, let
  $U$ be the weights realizing $f$ with
  $\gamma^d_{p,q}(f)=\gamma_{p,q}(U)$. Then $\sqrt[d]{\alpha}U$
  realizes $\alpha f$ establishing $\gamma^d_{p,q}(\alpha f) \leq
  \gamma_{p,q}(\sqrt[d]{\alpha}U) = \alpha \gamma_{p,q}(U)=\alpha
  \gamma^d_{p,q}(U)=\alpha \gamma_{p,q}^d(f)$. This establishes the
 non-negative homogeneity of $\gamma_{p,q}^d$.
\paragraph{Convex sublevel sets and homogeneity imply triangular inequality}
Let $\alpha(f)$ be non-negative homogeneous and assume that every
sublevel set $\{f\in\calN^d:\alpha(f)\leq \gamma\}$ is convex. Then
for  $f_1,f_2\in\calN^d$, defining $\gamma_1\defeq\alpha(f_1)$,
 $\gamma_2\defeq \alpha(f_2)$, $\tilde{f}_1\defeq(\gamma_1+\gamma_2)f_1/\gamma_1$, and
 $\tilde{f}_2\defeq(\gamma_1+\gamma_2)f_2/\gamma_2$, we have
\begin{align*}
\alpha(f_1+f_2)=\alpha\left(\frac{\gamma_1}{\gamma_1+\gamma_2}\tilde{f}_1+\frac{\gamma_2}{\gamma_1+\gamma_2}\tilde{f}_2\right)\leq
 \gamma_1+\gamma_2=\alpha(f_1)+\alpha(f_2).
\end{align*}
Here the inequality is due to the convexity of the
level set  and the fact that
$\alpha(\tilde{f}_1)=\alpha(\tilde{f}_2)=\gamma_1+\gamma_2$, because of
the homogeneity. Therefore $\alpha$ satisfies the triangular inequality
and thus it is a seminorm.

\end{proof}


\section{Path Regularization}\label{sec:path-appendix}
\subsection{Theorem \ref{thm:path-layer}}
\label{sec:path-layer}
\begin{lemma}\label{lem:unit-norm}
For any function $f\in \calN^{d,H}_{\gamma_{p,\infty} \leq \gamma}$ there is a layered network with weights $w$ such that $\gamma_{p,\infty}(w) = \gamma^{d,H}_{p,\infty}(f)$ and for any internal unit $v$, $\sum_{(u\rightarrow v)\in E} |w(u\rightarrow v)|^p = 1$.
\end{lemma}
\begin{proof}
Let $w$ be the weights of a network such that $\gamma_{p,\infty}(w) = \gamma^{d,H}_{p,\infty}(f)$. We now construct a network with weights $\widetilde{w}$ such that $\gamma_{p,\infty}(w) = \gamma^{d,H}_{p,\infty}(f)$ and for any internal unit $v$, $\sum_{(u\rightarrow v)\in E} |\widetilde{w}(u\rightarrow v)|^p = 1$. We do this by an incremental algorithm. Let $w_0=w$. At each step $i$, we do the following. 

Consider the first layer, Set $V_k$ to be the set of neurons in the layer $k$.
Let $x$ be the maximum of $\ell_p$ norms of input weights to each neuron in set $V_1$ and let $U_x\subseteq V_1$ be the set of neurons whose $\ell_p$ norms of their input weight is exactly $x$. Now let $y$ be the maximum of $\ell_p$ norms of input weights to each neuron in the set $V_1\setminus U_x$ and let $U_y$ be the set of the neurons such that the $\ell_p$ norms of their input weights is exactly $y$. Clearly $y<x$. We now scale down the input weights of neurons in set $U_x$ by $y/x$ and scale up all the outgoing edges of vertices in $U_x$ by $x/y$ ($y$ cannot be zero for internal neurons based on the definition). It is straightforward that the new network realizes the same function and the $\ell_{p,\infty}$ norm of the first layer has changed by a factor $y/x$. Now for every neuron $v\in V_2$, let $r(v)$ be the $\ell_{p}$ norm of the new incoming weights divided by $\ell_p$ norm of the original incoming weights. We know that $r(v)\leq x/y$. We again scaly down the input weights of every$v\in V_2$ by $1/r(v)$ and scale up all the outgoing edges of $v$ by $r(v)$. Continuing this operation to on each layer, each time we propagate the ratio to the next layer while the network always realizes the same function and for each layer $k$, we know that for every $v\in V_k$, $r(v)\leq x/y$. After this operation, in the network, the $\ell_{p,\infty}$ norm of the first layer is scaled down by $y/x$ while the $\ell_{p,\infty}$ norm of the last layer is scaled up by at most $x/y$ and the $\ell_{p,\infty}$ norm of the rest of the layers has remained the same. Therefore, if $w_i$ is the new weight setting, we have $\gamma_{p,\infty}(w_i) \leq \gamma_{p,\infty}(w_{i-1})$.

After continuing the above step at most $|V_1|-1$ times, the $\ell_p$ norm of input weights is the same for all neurons in $V_1$. We can then run the same algorithm on other layers and at the end we have a network with weight setting $\widetilde{w}$ such that the for each $k<d$, $\ell_p$ norm of input weight to each of the neurons in layer $k$ is equal to each other and $\gamma_{p,\infty}(\widetilde{w})\leq \gamma_{p,\infty}(w)$. This is in fact an equality because weight setting $w'$ realizes function $f$ and we know that $\gamma_{p,\infty}(w) = \gamma^{d,H}_{p,\infty}(f)$. A simple scaling of weights in layers gives completes the proof.
\end{proof}

\begin{reptheorem}{thm:path-layer}
For $p\geq 1$, any $d$ and (finite or infinite) $H$, for any $f\in\calN^{d,H}$:  $\pathr_p^{d,H}(f) = \gamma^{d,H}_{p,\infty}$.
\end{reptheorem}
\begin{proof}
By the Lemma \ref{lem:unit-norm}, there is a layered network with weights $\widetilde{w}$ such that $\gamma_{p,\infty}(\widetilde{w}) = \gamma^{d,H}_{p,\infty}(f)$ and for any internal unit $v$, $\sum_{(u\rightarrow v)\in E} |\widetilde{w}(u\rightarrow v)|^p = 1$. Let $W$ be the weights of the layered network that corresponds to the function $\widetilde{w}$. Then we have:
\begin{align}
v_p(\tilde{w}) &= \left(
    \sum_{\vin[i]\overset{e_1}{\rightarrow}v_1\overset{e_2}{\rightarrow}v_2\cdots\overset{e_k}{\rightarrow}\vout} \prod_{i=1}^k \abs{\widetilde{w}(e_i)}^p \right)^\frac{1}{p}\\ \label{eq:uniteq1}
    &= \left(\sum_{i_{d-1}=1}^H \dots \sum_{i_1=1}^H \sum_{i_0=1}^D  \lvert W_d[i_{d-1}]\rvert^p \prod_{k=1}^{d-1} \lvert W_k[i_{k},i_{k-1}]\rvert^p\right)^{\frac{1}{p}}\\
&= \left(\sum_{i_{d-1}=1}^H \lvert W_d[i_{d-1}]\rvert^p \dots \sum_{i_1=1}^H  \lvert W_k[i_2,i_{1}]\rvert^p \sum_{i_0=1}^D \lvert W_k[i_1,i_{0}]\rvert^p\right)^{\frac{1}{p}}\\
&= \left(\sum_{i_{d-1}=1}^H \lvert W_d[i_{d-1}]\rvert^p \dots \sum_{i_1=1}^H  \lvert W_k[i_2,i_{1}]\rvert^p \right)^{\frac{1}{p}}\\
&= \left(\sum_{i_{d-1}=1}^H \lvert W_d[i_{d-1}]\rvert^p \dots \sum_{i_2=1}^H  \lvert W_k[i_3,i_{2}]\rvert^p \right)^{\frac{1}{p}}\\ \label{eq:uniteq2}
&= \left(\sum_{i_{d-1}=1}^H \lvert W_d[i_{d-1}]\rvert^p \right)^{\frac{1}{p}} = \ell_p(W_d) =\gamma_{p,\infty}(W)\\
\end{align}
where inequalities~\ref{eq:uniteq1} to \ref{eq:uniteq2} are due to the fact that the $\ell_p$ norm of input weights to each internal neuron is exactly 1 and the last equality is again because $\ell_{p,\infty}$ of all layers is exactly 1 except the layer $d$.
\end{proof}
\subsection{Proof of Theorem \ref{thm:path-dag}}
\label{sec:proof-path-dag}
In this section, without loss of generality, we assume that all the internal nodes in a DAG have incoming edges and outgoing edges because otherwise we can just discard them. Let $\dout(v)$ be the longest directed path from vertex $v$ to $\vout$ and $\din(v)$ be the longest directed path from any input vertex $\vin[i]$ to $v$. We say graph $G$ is a sublayered graph if $G$ is a subgraph of a layered graph.

We first show the necessary and sufficient conditions under which a DAG is a sublayered graph.

\begin{lemma}\label{lem:layer-cond}
The graph $G(E,V)$ is a sublayered graph if and only if any path from input nodes to the output nodes has length $d$ where $d$ is the length of the longest path in $G$
\end{lemma}
\begin{proof}
Since the internal nodes have incoming edges and outgoing edges; hence if $G$ is a sublayered graph it is straightforward by induction on the layers that for every vertex $v$ in layer $i$, there is a vertex $u$ in layer $i+1$ such that $(v\rightarrow u)\in E$ and this proves the necessary condition for being sublayered graph.

To show the sufficient condition, for any internal node $u$, $u$ has $\din(v)$ distance from the input node in every path that includes $u$ (otherwise we can build a path that is longer than $d$). Therefore, for each vertex $v\in V$, we can place vertex $v$ in layer $\din(v)$ and all the outgoing edges from $v$ will be to layer $\din(v)+1$.
\end{proof}

\begin{lemma}\label{lem:path-edge}
If the graph $G(E,V)$ is not a sublayered graph then there exists a directed edge $(u\rightarrow v)$ such that  
$\din(u)+\dout(v)<d-1$ where $d$ the length of the longest path in $G$.
\end{lemma}
\begin{proof}
We prove the lemma by an inductive argument. If $G$ is not sublayered, by lemma~\ref{lem:layer-cond}, we know that there exists a path $v_0\rightarrow\dots v_i\dots \rightarrow v_{d'}$ where $v_0$ is an input node ($\din(v_0)=0$), $v_{d'}=\vout$ ($\dout(v_{d'}=0$) and $d'<d$. Now consider the vertex $v_1$. We need to have $\dout(v_1)=d-1$ otherwise if $\dout(v_1)<d-1$ we get $\din(u)+\dout(v)<d-1$ and if  $\dout(v_1)>d-1$ there will be path in $G$ that is longer than $d$. Also, since $\dout(v_1)=d-1$ and the longest path in $G$ has length $d$, we have $\din(v_1)=1$. 

By applying the same inductive argument on each vertex $v_i$ in the path we get $\din(v_i)=i$ and $\dout(v_i)=d-i$. Note that if the condition $\din(u)+\dout(v)<d-1$ is not satisfied in one of the steps of the inductive argument, the lemma is proved. Otherwise, we have $\din(v_{d'-1})=d'-1$ and $\dout(v_{d'-1})=d-d'+1$ and therefore $\din(v_{d'-1})+\dout(\vout) = d'-1<d-1$ that proves the lemma.
\end{proof}
\begin{reptheorem}{thm:path-dag}
  For any $p\geq 1$ and any $d$: $\displaystyle \gamma^d_{p,\infty}(f) =
  \min_{\textrm{$G\in \DAG(d)$}} \pathr^G_p(f)$.
\end{reptheorem}
\begin{proof}
Consider any $f_{G,w} \in \calN^{\DAG(d)}$ where the graph $G(E,V)$ is not sublayered. Let $\rho$ be the total number of paths from input nodes to the output nodes. Let $T$ be sum over paths of the length of the path. We indicate an algorithm to change $G$ into a sublayered graph $\tilde{G}$ of depth $d$ with weights $\tilde{w}$ such that $f_{G,w}=f_{\tilde{G},\tilde{w}}$ and $\pathr(w)=\pathr(\tilde{w})$. Let $G_0=G$ and $w_0=w$. 

At each step $i$, we consider the graph $G_{i-1}$. If $G_{i-1}$ is sublayered, we are done otherwise by lemma \ref{lem:path-edge}, there exists an edge $(u\rightarrow v)$ such that $\din(u)+\dout(v)<d-1$. Now we add a new vertex $\tilde{v}_i$ to graph $G_{i-1}$, remove the edge $(u\rightarrow v)$, add two edges $(u\rightarrow \tilde{v}_i)$ and $(\tilde{v}_i\rightarrow v)$ and return the graph as $G_{i}$ and since we had $\din(u)+\dout(v)<d-1$ in $G_{i-1}$, the longest path in $G_{i}$ still has length $d$. We also set $w(u\rightarrow \tilde{v}_i) = \sqrt{|w(u\rightarrow v)|}$ and $w(\tilde{v}_i\rightarrow v) = \sign(w(u\rightarrow v))\sqrt{|w(u\rightarrow v)|}$. Since we are using rectified linear units activations, for any $x>0$, we have $[x]_+=x$ and therefore:
\begin{align*}
w(\tilde{v}_i\rightarrow v)\left[w(u\rightarrow \tilde{v}_i) o(u)\right]_+ &=\sign(w(u\rightarrow v))\sqrt{|w(u\rightarrow v)|}\left[\sqrt{|w(u\rightarrow v)|} o(u)\right]_+\\
&=\sign(w(u\rightarrow v))\sqrt{|w(u\rightarrow v)|}\sqrt{|w(u\rightarrow v)|} o(u)\\
&=w(u\rightarrow v)o(u)
\end{align*}
So we conclude that $f_{G_{i},w_i}=f_{G_{i-1},w_{i-1}}$. Clearly, since we didn't change the length of any path from input vertices to the output vertex, we have $\pathr(w)=\pathr(\tilde{w})$. Let $T_i$ be sum over paths of the length of the path in $G_i$. It is clear that $T_{i-1} \leq T_{i}$ because we add a new edge into a path at each step. We also know by lemma~\ref{lem:layer-cond} that if $T_{i}=\rho d$, then $G_i$ is a sublayered graph. Therefore, after at most $\rho d - T_0$ steps, we return a sublayered graph $\tilde{G}$ and weights $\tilde{w}$ such that $f_{G,w}=f_{\tilde{G},\tilde{w}}$. We can easily turn the sublayered graph $\tilde{G}$ a layered graph by adding edges with zero weights and this together with Theorem \ref{thm:path-layer} completes the proof.
\end{proof}

\section{Hardness of Learning Neural Networks}\label{sec:hardness}

 \citet{Daniely14} show in Theorem 5.4 and in Section 7.2 that subject to the strong random CSP assumption, for
  any $k=\omega(1)$ the hypothesis class of intersection of
  homogeneous halfspaces over $\{\pm 1\}^n$ with normals in $\{\pm
  1\}$ is not efficiently PAC learnable (even
  improperly)\footnote{Their Theorem 5.4 talks about unrestricted
    halfspaces, but the construction in Section 7.2 uses only data in
    $\{ \pm 1 \}^D$ and halfspaces specified by $\langle w,x\rangle
    >0$ with $w\in\{\pm 1\}^D$}. Furthermore, for any $\epsilon>0$, \cite{klivans2006} prove this hardness result subject to intractability of $\tilde{Q}(D^{1.5})$-unique shortest
vector problem for $k=D^\epsilon$.

  If it is not possible to efficiently PAC learn intersection of
halfspaces (even improperly), we can conclude it is also not possible
to efficiently PAC learn any hypothesis class which can represent such
intersection. In Theorem~\ref{thm:hardness} we show that intersection of homogeneous half spaces can be realized with unit margin by neural networks with bound norm.
\begin{theorem}\label{thm:hardness}
For any $k>0$, the intersection of $k$ homogeneous half spaces is realizable with unit margin by $\calN^2_{\gamma_{p,q} \leq \gamma}$ where $\gamma=4D^{\frac{1}{p}}k^2$.
\end{theorem}

\begin{proof}
The proof is by a construction that is similar to the one in \cite{livni14}.
For each hyperplane $\inner{w_i}{x}>0$, where
  $w_i\in \{\pm 1\}^D$, we include two units in the first layer:
  $g^+_{i}(x) = [\inner{w_i}{x}]_+$ and $g^-_i(x) =
  [\inner{w_i}{x}-1]_+$.  We set all incoming weights of the
  output node to be $1$.  Therefore, this network is realizing the
  following function:
$$
f(x) = \sum_{i=1}^k \left([\inner{w_i}{x}]_+ - [\inner{w_i}{x}-1]_+\right)
$$
Since all inputs and all weights are integer, the outputs of the first
layer will be integer, $\left([\inner{w_i}{x}]_+ -
  [\inner{w_i}{x}-1]_+\right)$ will be zero or one, and $f$ realizes
the intersection of the $k$ halfspaces with unit margin. Now, we just
need to make sure that $\gamma^2_{p,q}(f)$ is bounded by $\gamma=4D^{\frac{1}{p}}k^2$:
\begin{align*}
\gamma^2_{p,q}(f) &= D^{\frac{1}{p}} (2k)^{\frac{1}{q}}(2k)^{\frac{1}{p}} \\
&\leq D^{\frac{1}{p}}(2k)^{2} = \gamma.
\end{align*}
\end{proof}
\end{document}